\documentclass[twoside]{article}
\usepackage[accepted]{aistats2014}
\usepackage{amsmath}
\usepackage{amssymb}
\usepackage{amsthm}
\usepackage{verbatim}
\usepackage{alltt}
\usepackage{fullpage}
\usepackage{cite}
\usepackage{hyperref}
\usepackage{graphicx}
\usepackage{wrapfig}

\theoremstyle{plain}
\newtheorem{theorem}{Theorem}
\newtheorem{lemma}[theorem]{Lemma}

\newcommand \tophrule {\vspace{2mm}\rule{\columnwidth}{0.8pt}\vspace{-1mm}}
\newcommand \bothrule {\vspace{-3mm}\rule{\columnwidth}{0.4pt}\vspace{2mm}}
\newenvironment{code}
  { \begin{minipage}{\columnwidth} \tophrule \begin{alltt}\sf }
  { \end{alltt} \bothrule \end{minipage} }
\newcommand \codebf [1] {{\bfseries{}#1}}
\newcommand \codeit [1] {{\itshape{}#1}}

\begin{document}

\ifnum\statePaper=1{

\twocolumn[
\aistatstitle{Scaling Nonparametric Bayesian Inference via Subsample-Annealing}

\aistatsauthor{ Fritz Obermeyer \And Jonathan Glidden \And Eric Jonas }
\aistatsaddress{ Salesforce.com \And  Salesforce.com \And  Salesforce.com }
]

}\else{

\twocolumn[
\aistatstitle{Scaling Nonparametric Bayesian Inference via Subsample-Annealing}
\aistatsauthor{ Anonymous Author }
\aistatsaddress{ Anonymous Organization }

]

}\fi

\begin{abstract}
We describe an adaptation of the simulated annealing algorithm to nonparametric clustering and related probabilistic models.
This new algorithm learns nonparametric latent structure over a growing and constantly churning subsample of training data,
where the portion of data subsampled can be interpreted as the inverse temperature $\beta(t)$ in an annealing schedule.
Gibbs sampling at high temperature (i.e., with a very small subsample) can more quickly explore sketches of the final latent state by
(a) making longer jumps around latent space (as in block Gibbs) and
(b) lowering energy barriers (as in simulated annealing).
We prove subsample annealing speeds up mixing time $N^2\to N$ in a simple clustering model and $\exp(N)\to N$ in another class of models, where $N$ is data size.
Empirically subsample-annealing outperforms naive Gibbs sampling in accuracy-per-wallclock time,
and can scale to larger datasets and deeper hierarchical models.
We demonstrate improved inference on million-row subsamples of US Census data and network log data and a 307-row hospital rating dataset,
using a Pitman-Yor generalization of the Cross Categorization model.
\end{abstract}

\section{Introduction}

\begin{figure}[t]
\begin{tabular}{c c c}
\hspace{0.5em} {\small\sf Energy} &
\hspace{1.2em} {\small\sf Simulated} &
\hspace{1.2em} {\small\sf Subsample} \\[-1mm]
\hspace{0.5em} {\small\sf Landscape} &
\hspace{1.2em} {\small\sf Annealing} &
\hspace{1.2em} {\small\sf Annealing} \\
\hspace{-4mm}
\includegraphics[height=2.4cm]{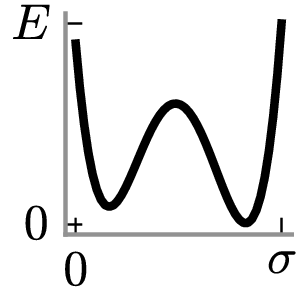}
\hspace{-3mm}
&
\hspace{-3mm}
\includegraphics[height=2.4cm]{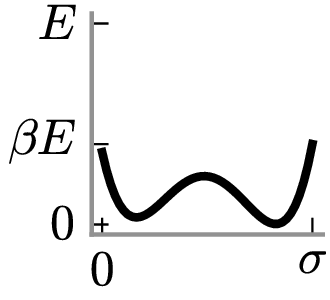}
\hspace{-3mm}
&
\hspace{-3mm}
\includegraphics[height=2.4cm]{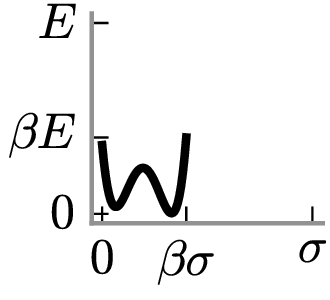}
\hspace{-3mm}
\end{tabular}
\vspace{-3mm}
\caption{
Simulated annealing vertically compresses the energy landscape, providing exponentially faster inter-mode mixing.
Subsample annealing additionally horizontally compresses the energy landscape, providing quadratically faster local mixing. ($E$ is energy scale, $\sigma$ is diffusion scale)
}
\label{fig:landscape}
\end{figure}

Recently there has been a flourishing of discrete nonparametric Bayesian models, extending clustering models such as the Dirichlet Process Mixture to more exotic nonparametric models such as the Indian Buffet Process, Cross Categorization, and infinite Hidden Markov Models\cite{ghahramani2013bayesian}.
At the same time, there has been rapid progress in scaling continuous Bayesian models to larger structured datasets, including results in stochastic gradient descent (SGD), \cite{Bottou08thetradeoffs} and a trend towards approximate inference that trades accuracy for speed.
Inference in discrete models is lagging.
Scalable variational inference methods can often be found in particular models, but these methods are more problem-specific than, say, Gibbs-sampling or SGD with minibatches.

The single-site Gibbs sampler is an easy-to-implement MCMC learning algorithm that is applicable to a wide range of discrete models.
The contribution of this paper is a time-inhomogeneous-MCMC extension of the single-site Gibbs sampler that is still easy to implement yet has been found to scale well in data size and model complexity.
Our extension runs a standard Gibbs-sampler on a subsample of the dataset, incrementally churning datapoints in and out of the subsample and progressively growing the subsample to include all datapoints.
The only extra parameter to tune is the schedule of subsample sizes.
To borrow common terminology from the simulated annealing literature, we call this new method \emph{subsample annealing}, and treat the portion of data present in the subsample at time $t$ as the inverse temperature $\beta(t)$.
Indeed deeper mathematical connections indicate that subsample annealing is approximately equivalent to simultaneously annealing on (a) the energy (as in classical simulated annealing), (b) the metric or stepsize (in a Langevin-dynamics limit), and (c) the regularization weight of hyperparameter priors.

Like simulated annealing, subsample annealing moves towards an increasingly accurate approximate solution, while guaranteeing that in the long-schedule limit, the final sample is drawn from the true posterior%
\footnote{\textit{Proof:} assuming a linear annealing schedule, the last $\left\lfloor T/N\right\rfloor$ of $T$ iterations use all $N$ datapoints.}.
In this paper we provide both theoretical evidence (in section~\ref{sec:analysis}) and empirical evidence (in section~\ref{sec:experiments}) that this early fast approximate inference results in more accurate final samples, compared to those from finite-time full-data MCMC.
This speedup is observable in units of quality-per-Gibbs-step, but even more dramatically in units of quality-per-wall-clock-time, since Gibbs steps are cheaper with less data.

We begin in section~\ref{sec:algorithm} with a description of the algorithm.
In section~\ref{sec:analysis} we examine a simple toy model and analyze asymptotic convergence rate, showing in section~\ref{sec:analysis-global} that annealing can offer exponential speedup in some cases.
In section~\ref{sec:experiments} we describe experimental results learning a more complicated Cross-Categorization model of three datasets with tens of features and up to 1 million datapoints.

\section{Inference via Subsample-Annealing}
\label{sec:algorithm}

Subsample annealing generalizes the Gibbs-sampler most commonly used in structural latent variable models.
Our approach considers the two steps of the Gibbs-sampler, that is \emph{forgetting}/\emph{removing} a datapoint and \emph{conditionally-sampling}/\emph{assigning} a datapoint, as separate operations.
Subsample annealing is simply the implementation of a generalized schedule for these two operations.

We use the Dirichlet Process Mixture Model (DPMM) as the motivating example in this section, but stress that this approach generalizes to a wide class of models in which the latent state is an assignment of datapoints to combinatorial objects.

\subsection{The DPMM Model}
\label{sec:dpmm}
The DPMM is a popular nonparametric clustering model in which we learn the posterior distribution on assignments of datapoints to clusters.
For a detailed review of this model see \cite{teh2006hierarchical}.

Let $X=[X_1,\dots,X_N]$ be a list of datapoints, and $S=\{1,\dots,N\}$ be the set of datapoint indices.
We represent the latent assignment of datapoints to clusters as a partition $\pi\subseteq 2^S$ of datapoints into disjoint subsets.

The Chinese Restaurant process (CRP) representation of the DP gives a simple way to describe the generative process.
Let $\alpha > 0$ be a parameter, called the concentration parameter, $G(\theta)$ a prior on cluster parameters, called the base distribution, and $F(-|\theta)$ a likelihood model for observations given cluster parameters, called the component model.
The first datapoint is assigned to its own cluster and $\theta_1 \sim G, \; X_1 \sim F(-\mid \theta_1)$.
Recursively, let $\pi= \{\pi_k\}_{k=1}^{K}$ be a partition of the first $n-1$ indices 
and $\{\theta_k\}_{k=1}^{K}$ a set of parameters.
We generate datapoint $X_n$ as follows:
with probability $\propto \#\pi_k$, assign datapoint $n$ to cluster $k$ and draw $X_n \sim F(-\mid \theta_k)$;
with probability $\propto \alpha$, assign datapoint $n$ to a new cluster $K+1$ and draw $\theta_{K+1} \sim G,\; X_n \sim F(-\mid \theta_{K+1})$.

Importantly, the distribution on clusterings induced by the CRP is \emph{exchangeable}, i.e. invariant to index ordering.

\subsection{Gibbs Sampling in DPMMs}
\label{sec:dpmmgibbs}

The exchangeability property of the CRP suggests a simple Gibbs-sampling algorithm for sampling from the posterior distribution on clusterings \cite{neal2000markov}.
First, remove a random%
\footnote{We assume a random scanning schedule which is easier to analyze and not inferior to systematic scanning \cite{diaconis2008Gibbs}} 
datapoint $X_{\ast}$ (with $\ast$ drawn uniformly from $S$) to form a set $S'\gets S\setminus\{\ast\}$ and a restricted partition $\pi'$ of $S'$.
By exchangeability, the conditional distribution on assignments of $\ast$ is the same as if it were last in the index.
If $G$ and $F$ are conjugate, we can integrate out the cluster parameters and sample a new assignment

\begin{equation*}
    \label{condassignprob}
    \begin{split}
      &P[\ast\text{ adds to cluster }k]
      \propto \#\pi'_k\; p(X_{\ast} \mid X_{\pi'_k}) \\
      &P[\ast\text{ starts a new cluster}]
      \propto \alpha\; p(X_{\ast})
    \end{split} \tag{$\star$}
\end{equation*}

Where $p(X_\ast \mid X_{\pi'_k})$ is the marginal likelihood.

In the case where component models are conjugate, the latent state can be represented entirely by the assignment vector and the algorithm described is referred to as the \emph{collapsed} Gibbs-sampler.

\subsection{Subsample Annealing}

The key observation of this paper is that by decoupling the \emph{remove} and \emph{assign} parts of a Gibbs sampler, we can do approximate inference with proper subsamples $S$ of data points.
That is, after removing datapoint $r$, we add a datapoint $a$ possibly distinct from $r$.

Let $S$ be a possibly proper subset of the data indices, that is $S \subseteq\{1, \dots,N\}$.
We remove a random datapoint $r$ as above and form the set $S'$ and restricted partition $\pi'$.
We now draw a random datapoint $X_a$ not in $S'$ ($a$ drawn from $\{1,\dots,N\}\setminus S'$) to assign.
In the full data case the just-removed datapoint $\{r\}=\{1,\dots,N\}\setminus S'$ is always immediately reassigned, so we recover the classic Gibbs-sampling algorithm.
Finally we conditionally sample the partition assignment of $a$ according to \eqref{condassignprob}.

Typically the assignment step dominates the computational cost of the algorithm.
Important in nonparametric models, the number of clusters generally grows as the subsample size increases (logarithmically in the CRP, polynomially in the Pitman-Yor process), so the sampling step becomes more expensive as the subsample size grows.

\subsection{Subsample Annealing Schedules}
\label{sec:schedules}

Generally we can consider any subsample size schedule $\beta(t)N=\#S_t$ satisfying $|\beta(t+1)-\beta(t)|N\in\{+1,-1\}$.
For example the standard Gibbs sampler follows a constant schedule:
\begin{code}
    \codebf{Algorithm} \codeit{Prior+Gibbs inference strategy}
    1. Generate initial clustering of \(N\) datapoints.
    2. \codebf{for} \(t\) in \([1, ..., T]\):
    3.     \codebf{for} \(n\) in \([1, ..., N]\):
    4.         Remove a random assigned datapoint.
    5.         Reassign that datapoint.
\end{code}
All assignments are done with respect to the conditional distribution \eqref{condassignprob}.
Generally the initial clustering of the data is a draw from the prior.

Another strategy is to initialize by assigning datapoints sequentially, then run the Gibbs sampler on the full dataset.
That is, we could start by adding datapoints incrementally, sampling each assignment conditioned on all previous assignments.
In the subsample annealing frame this a two-part quench-then-mix schedule,
quickly ramping from the empty subsample to the full dataset, then running the Gibbs sampler for a long time at full dataset size:
\begin{code}
    \codebf{Algorithm} \codeit{Sequential+Gibbs inference strategy}
    1. Initialize empty.
    2. \codebf{for} \(n\) in \([1, ..., N]\):
    3.     Pick a random unassigned datapoint.
    4.     Assign the unassigned datapoint.
    5. \codebf{for} \(t\) in \([1, ..., T-1]\):
    6.     \codebf{for} \(n\) in \([1, ..., N]\):
    7.         Remove a random assigned datapoint.
    8.         Reassign that datapoint.
\end{code}
With little more coding effort one can implement a gradual linear subsample-annealing schedule
\begin{code}
    \codebf{Algorithm} \codeit{Anneal Subsample inference strategy}
    1. Initialize empty.
    2. \codebf{for} \(n\) in \([1, ..., N]\):
    3.     Pick a random unassigned datapoint.
    4.     Assign the unassigned datapoint.
    5.     \codebf{for} \(t\) in \([1, ..., T]\):
    6.         Remove a random assigned datapoint.
    7.         Pick a random unassigned datapoint.
    8.         Assign the unassigned datapoint.
\end{code}
where datapoints are gradually churned in and out of the latent state.

We have found in practice that this linear-growth schedule $\beta(t)\approx\frac t T$ performs well, and we analyze this schedule henceforth.
In choosing a schedule for real inference, it is important to consider the wall clock time cost of each iteration, which in practice differs from the Gibbs assignment count.
For example in CRP models, the cost of Gibbs assignment grows log(subsample size).

Like the Gibbs sampler, Subsample annealing composes well with other inference methods, e.g., for hyperparameter and structure inference.
\begin{wrapfigure}{r}{0.5\columnwidth}
    \vspace{-2mm}
    \hspace{-2mm}
    \includegraphics[width=0.5\columnwidth]{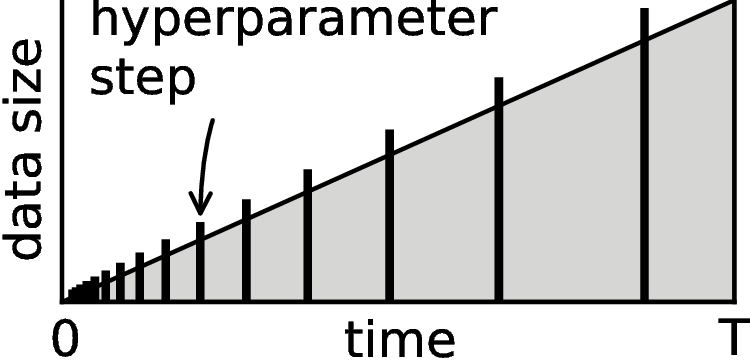}
    \vspace{-3mm}
\end{wrapfigure}
In practice we schedule hyperparameter inference steps to happen once per full cycle through the current subsample.
Early in the schedule, data churns relatively faster, so hyperparameter inference steps are performed more often.

%
%
%
%

\section{Convergence Analysis}
\label{sec:analysis}

We now analyze in detail the effect of subsample-annealing in two simple toy models, building intuition and proving asymptotic speedup.
Specifically we show that
(a) in a clustering model, subsample annealing speeds up mixing time $N^2\log(\frac 1 \epsilon)\to N\log(\frac 1 \epsilon)$, and
(b) in a simple two-mode energy barrier model, speeds up mixing time $\exp(N)\log(\frac 1 \epsilon)\to N\operatorname{poly}(\frac 1 \epsilon)$.

\subsection{Algorithm Interpretation}
\label{sec:analysis-interpretation}

One way to see the effect of subsample annealing is via the geometric interpretation of ``classical'' simulated annealing, in Figure~\ref{fig:landscape}.
While classical annealing vertically compresses the energy landscape by a factor $\beta$, subsample annealing additionally%
\footnote{in many cases, however the first toy model below is an exception.}
horizontally compresses the energy landscape by the same factor $\beta$,
corresponding in the Langevin dynamics limit to a quadratic $\beta^{-2}$ diffusion speedup%
\footnote{Langevin dynamics is invariant under the transformation $(x,t)\mapsto(\beta x,\beta^2 t)$.}%
.
Thus the two effects of subsample annealing are (a) local quadratic speedup, and (b) classical simulated annealing, which allows faster mixing between modes, as shown below.

\subsection{Quadratic Local Speedup}
\label{sec:analysis-local}

Our first toy model distills the space-compressing behavior of subsample annealing.
Although in many clustering models, energy scales linearly with data size, we choose for the moment a single-feature model where the data probability exactly balances the prior, yielding no energy scaling, and hence simplifying asymptotic analysis.

Consider a two-component mixture model of boolean data, which we shall think of as a balls-in-urns model.
With known probability $p$ balls are either generated from the \verb$left$ or \verb$right$ urn.
The left urn generates balls which are \verb$red$ with unknown probability $p_l$ and \verb$blue$ with probability $1-p_l$.
The right urn does the same according to parameter $p_r$.
After generating $N$ balls, we observe the balls and their colors.
We assume a beta prior on $p_l$ and $p_r$ with hyperparameter $\alpha>0$.
We are interested in the posterior assignment of balls to urns.

We integrate out $p_l$ and $p_r$ by conjugacy and 
since datapoints of the same color are indistinguishable, we project down to the equivalent inference problem of inferring the counts of red and blue balls in the left and right urn from the total red and blue counts.

Intuitively the posterior is multimodal, as the beta priors prefer segregation.
Most likely all the red balls came from the right urn or they all came from the left urn.
Figure~\ref{fig:example2-fig2} shows the distribution of latent states after different inference strategies of equal time-cost, in order of decreasing total variational distance (TVD) from the true posterior, which is indeed multimodal.
The parameters settings shown are the Jeffreys prior $\alpha=\frac 1 2$, and a slight bias $p=0.45$ towards one urn.
The data consists of 8000 red balls and 12000 blue balls.

To understand the effects of dataset scale when subsampling large datasets, we analyze the large-data continuous limit of the subsample-annealed Gibbs-sampler, where inference can be modeled by Langevin dynamics.

Consider a growing dataset with a fixed ratio of red-to-blue datapoints (2:3 in Figure~\ref{fig:example2-fig2}).
Parametrize the latent state by intrinsic variables
\begin{align*}
  x &= \frac{\#\text{red on left}}{\#\text{red total}},
& y &= \frac{\#\text{blue on left}}{\#\text{blue total}}
\end{align*}
\begin{lemma}
\label{thm:continuum}
In the continuum limit as $N\to\infty$ with the proportion of red balls fixed to $r\in(0,1)$,
the single-site Gibbs sampler's effect can be described by a Fokker-Planck PDE \cite{gardiner2009stochastic}
\begin{align*}
  \frac \partial {\partial t} p(x,y,t)
  &= -\frac 1 N \nabla \left( f(x,y,t,\frac \alpha N, r) p(x,y,t) \right) \\
  &+ \frac 1 {2 N^2} \nabla^2 \left( D(x,y,t,\frac \alpha N, r) p(x,y,t) \right)
\end{align*}
with drift vector $f$ and diffusion matrix $D$ depending only on intrinsic quantities, invariant of dataset size.
\end{lemma}
(see Appendix~\ref{appendix:local} for proof)
As $N$ grows and $\alpha/N$ is held fixed, the diffusion rate scales as $1/N^2$.
Thus early in the schedule, subsample annealing mixes quadratically faster; most of the mixing happens early in the schedule.

\begin{theorem}
\label{thm:local}
Consider a two-urn model with $\alpha > 0$ and $p\in(0,1)$ fixed, and a constant ratio of red:blue balls.
To bound total variational distance below  $\epsilon\in(0,1)$,
(a) cold inference ($\beta = 1$) requires time
$$
  T_{\text{cold}} = O \left( N^2 \log\left(1/\epsilon\right) \right),
$$
and (b) annealing at schedule $\beta = \frac t T$ requires time
$$
  T_{\text{anneal}} = O \left( N \log\left(1/\epsilon\right) \right).
$$
\end{theorem}
(see Appendix~\ref{appendix:local} for proof)
Thus in the absence of data-linear energy barriers between modes, subsample annealing scales linearly with data.
We shall see below that this data-linearity is preserved even in the presence of some energy barriers.

Inspired by the energy-and-space contraction result above, we developed a full-data inference strategy that behaves like subsample annealing in the continuum limit, by block Gibbs sampling on a $1/\beta$-sized hand full of uniformly-colored balls at a time, moving them from one urn to another.
This ``Anneal Stepsize'' strategy, measured in Figure~\ref{fig:example2-fig2},
is only of theoretical interest, being more expensive and difficult to generalize to other models.

\begin{figure*}[t]
\includegraphics[width=1.0\textwidth]{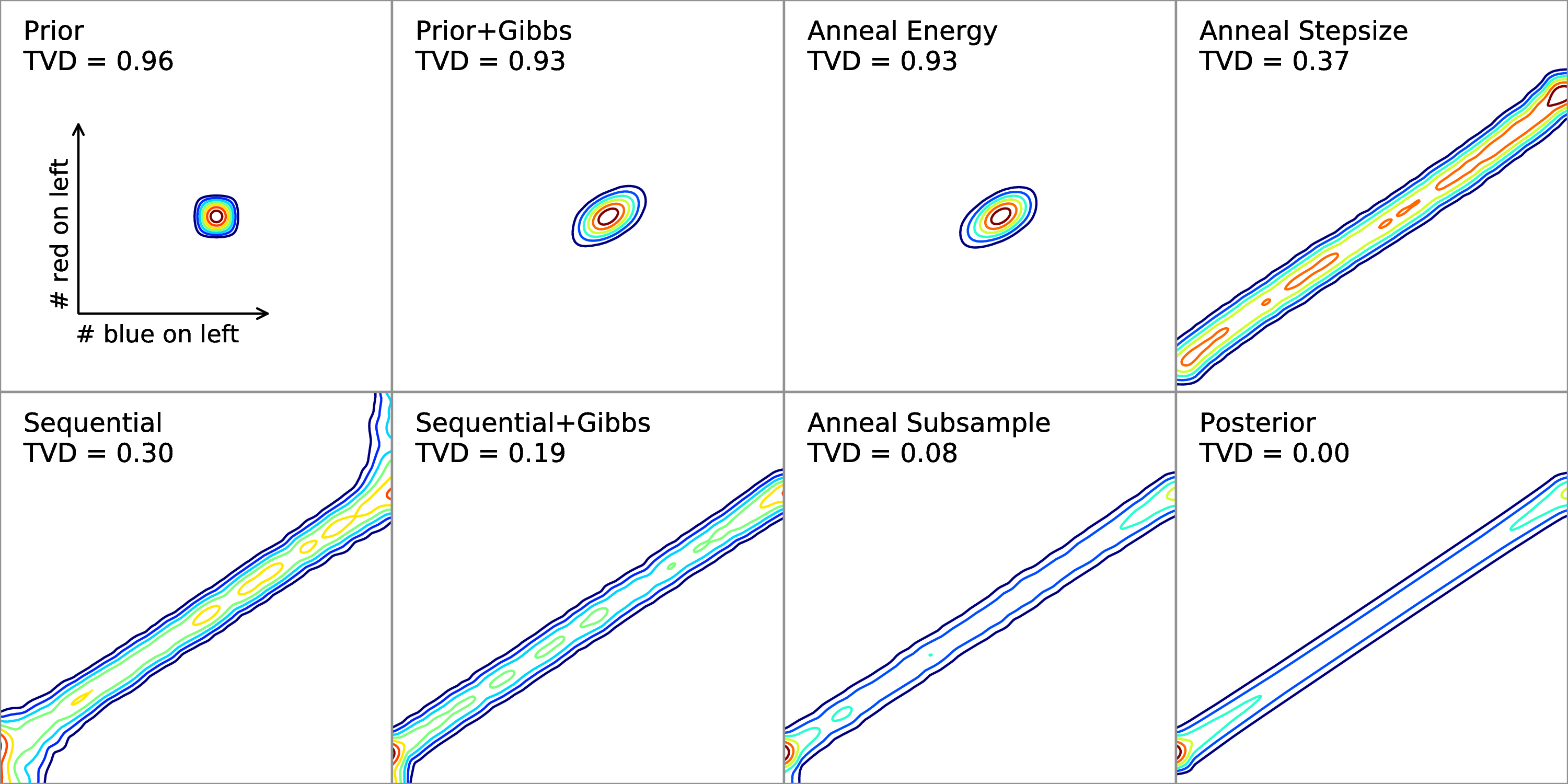}
\caption{
Distribution of latent states after $10N$ Gibbs steps with $N=20000$ datapoints (= 8000 red + 12000 blue).
Prior and posterior are exact; others are smoothed histograms of $10^6$ Monte Carlo samples.
Total variational distance (TVD) is computed after smoothing.
The Sequential strategy only performs $N$ initial Gibbs assignments.
The Sequential+Gibbs strategy performs the $N$ sequential assignments + $9N$ full-data Gibbs steps.
The Anneal Energy strategy runs Gibbs inference on full data but with a ``hot'' data likelihood.
The Anneal Stepsize strategy is a full-data inference strategy that behaves like subsample annealing in the continuum limit.
}
\label{fig:example2-fig2}
\end{figure*}

\subsection{Exponential Global Speedup}
\label{sec:analysis-global}

We next show that simulated annealing to temperature 1 can provide exponential speedup over naive MCMC in a simple model of a bimodal system.
In many models, energy barriers are proportional to dataset size, so subsample annealing behaves like classical simulated annealing (with a quadratic speedup, as above).

Our next toy model distills the behavior of inference around data-linear energy barriers.
The previous clustering model has two modes but was constructed to have a limiting energy constant in $N$, leading to the continuous limit.
Slight modifications lead to a bimodal system with limiting energy proportional to dataset size, for example: inferring cluster parameters rather than marginalizing over them; or learning multiple features; or using other conjugate feature models such as normal-inverse-$\chi^2$.

Consider a system with two modes separated by a low-probability barrier, for example a two-feature version of the above two-urn model.
Assuming within-mode mixing is much faster than inter-mode jumping,
we project the entire state down to the probability masses $[x, 1-x]$ in the two modes.
Let $\gamma N$ be the energy gap separating modes and $\delta N$ be the energy barrier between modes, both proportional to data size.
\begin{lemma}
\label{thm:dynamics}
The continuous-time dynamics of inference at inverse temperature $\beta$ is
\begin{align}
\label{eqn:dynamics}
  \frac {dx} {dt}
  &= \exp(-\beta\delta N) \left[ \frac 1 {1+\exp(-\beta\gamma N)} - x \right]
\end{align}
\end{lemma}
(see Appendix~\ref{appendix:global} for proof)
We are then interested in running inference from time 0 to some time $T$, starting from an arbitrary ``data blind'' initial state.
\begin{theorem}
\label{thm:global}
Assume the energy barrier is positive $\delta>0$ and the dataset is larger than necessary to observe the energy gap $N \gg \frac {\log(2/\epsilon)} \gamma$.
Then to bound total variational distance below  $\epsilon\in(0,1)$,
(a) cold inference ($\beta = 1$) requires time
$$
  T_{\text{cold}} = O
  \left(
    \exp(N\delta) \log\left(1/\epsilon\right)
  \right),
$$
and (b) annealing at schedule $\beta = \frac t T$ requires time
$$
  T_{\text{anneal}} = O
  \left(
    N \delta
    \log\left(1/\epsilon\right)
    \left(1/\epsilon \right)^{\frac \delta \gamma}
  \right).
$$
\end{theorem}
(see Appendix~\ref{appendix:global} for proof)
Thus in this toy model, simulated annealing to temperature 1 can achieve exponential speedup as data size grows, at least when resolving features $\gamma$ gross enough to be detectable already with a small subsample of data (or equivalently already at a high temperature).
Practically Theorem~\ref{thm:global} means that annealing with progressively longer durations $T$, we can resolve progressively finer features of the data, whereas cold inference resolves features at rates independent of their significance $\gamma$, i.e. depending only on $\delta$.
Thus simulated annealing gracefully degrades in quality depending on inference difficulty $\delta$ and feature significance $\gamma$; it learns a combination of easy-to-infer fine features and difficult-to-infer gross features.

Theorem~\ref{thm:global} concerns classical simulated annealing in a two state model.
To relate this exponential speedup back to subsample annealing in clustering models, observe that when subsampling, the dynamics in Equation~\ref{eqn:dynamics} would only speed up (by an extra factor of $N^{-2}$).
Hence at fixed feature grossness $\gamma$, subsample annealing also mixes exponentially faster.

\section{Experimental Results}
\label{sec:experiments}

To test the subsample annealing technique, we use a Pitman-Yor \cite{pitman1997two} extension of the Cross-categorization model \cite{mansinghka2009cross} which first partitions features via a Pitman-Yor process, and then models each set of features as a Pitman-Yor mixture of a product of features.
\begin{wrapfigure}{r}{0.48\columnwidth}
    \vspace{-5mm}
    \begin{center}
        {\sf \small A sample }\\[-1mm]
        {\sf \small cross-categorization}
        \includegraphics[width=0.48\columnwidth]{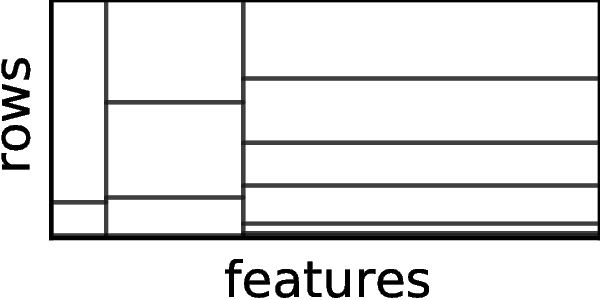}
    \end{center}
    \vspace{-4mm}
\end{wrapfigure}
Categorical features are modeled as a mixture of multinomials with a non-uniform Dirichlet prior, and real-valued features are modeled as a mixture of Gaussians with a normal-inverse-$\chi^2$ prior.
In addition to feature and data partitions, we learn Pitman-Yor hyperparameters $(\alpha,d)$ for each feature-set's partition of data and for the overall partition of features.
For each categorical feature we learn its Dirichlet hyperparameters, and for each real-valued feature we learn its normal-inverse-$\chi^2$ hyperparameters \cite{fink1997compendium}.
All hyperparameter priors are discrete grids spanning a wide range of values.

We learn hyperparameters by Gibbs-sampling conditioned on data-to-cluster assignments.
To learn the partitioning of features, we use a non-conjugate Gibbs-sampler and a Metropolis-Hastings sampler to propose new feature sets from the Pitman-Yor prior.
This last proposal step is sensitive to data size and appears to significantly benefit from subsample annealing.

We analyze three dataset:
a small 307-row hospital ranking dataset with 63 real-valued features \cite{wennberg2008tracking};
$10^4$--$10^6$ row subsamples of a US Census dataset \cite{Bache+Lichman:2013} with 68 categorical features; and
$10^4$--$10^6$ row subsamples of the KDD Cup 1999 network log dataset \cite{hettich99kddcup} with 40 features including both real-valued and categorical.


We assess inference quality by random cross-validation.
Each sample trains on a random subset of $7/8$ of the data and is scored on the remaining $1/8$ of data using the function
\[
  \sum_{t\in\text{test}} \log P(t\mid\text{trained model})
\]
To show outlying bad samples, we plot individual samples' crossvalidation scores as well as the mean log score for each algorithm.
To compare quality across datasets, we shift and scale log scores to be zero-mean unit-variance within each dataset.

\begin{figure*}[t!]
\begin{tabular}{c @{\hspace{15mm}} c}
\multicolumn{2}{c}{
\textsf{\small Hospitals 307$\times$63}
} \\
\multicolumn{2}{c}{
\includegraphics[width=0.41\textwidth]{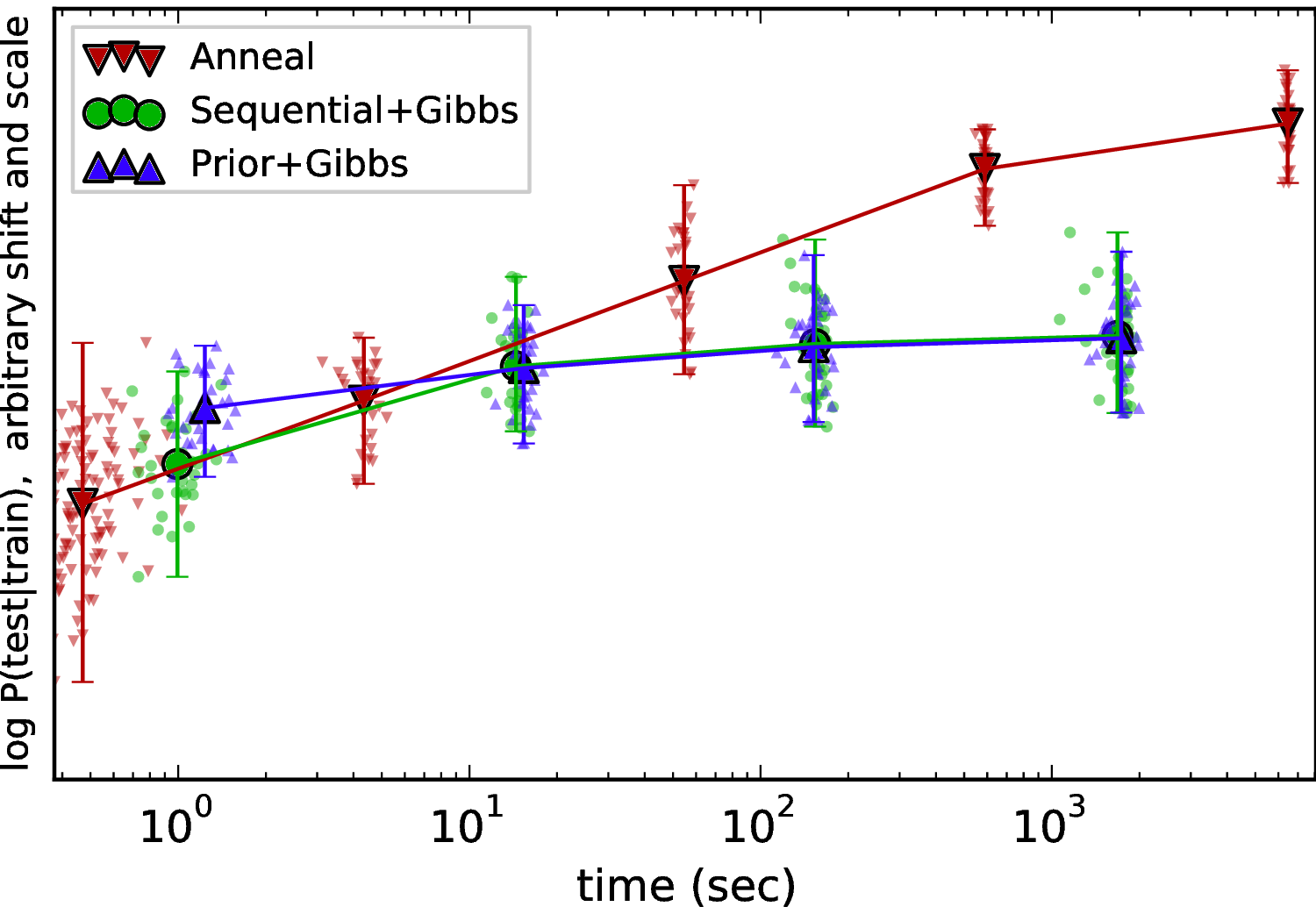}
} \\
\begin{tabular}{c}
\textsf{\small Network 10000$\times$40} \\
\includegraphics[width=0.41\textwidth]{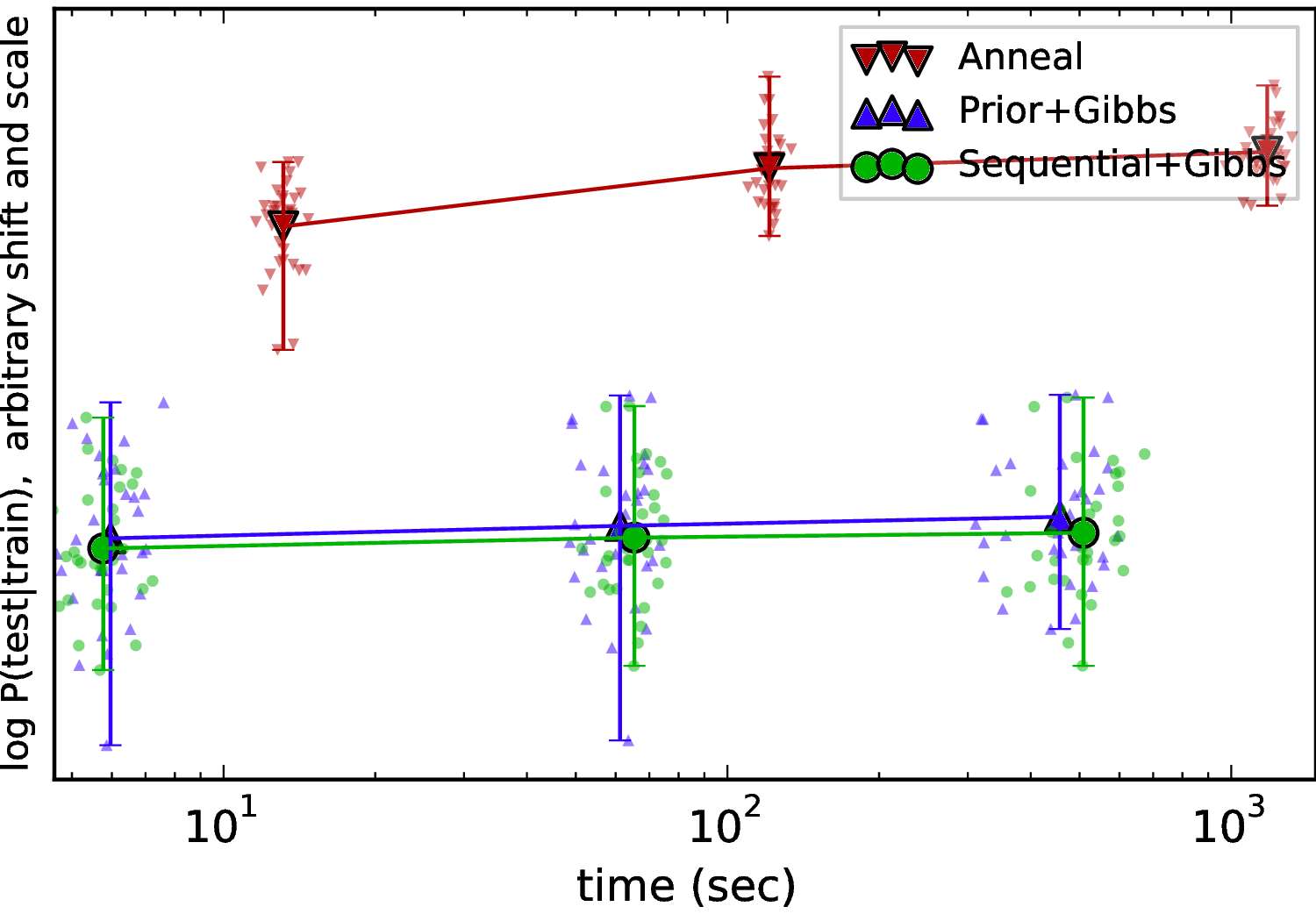} \\
\textsf{\small Network 100000$\times$40} \\
\includegraphics[width=0.41\textwidth]{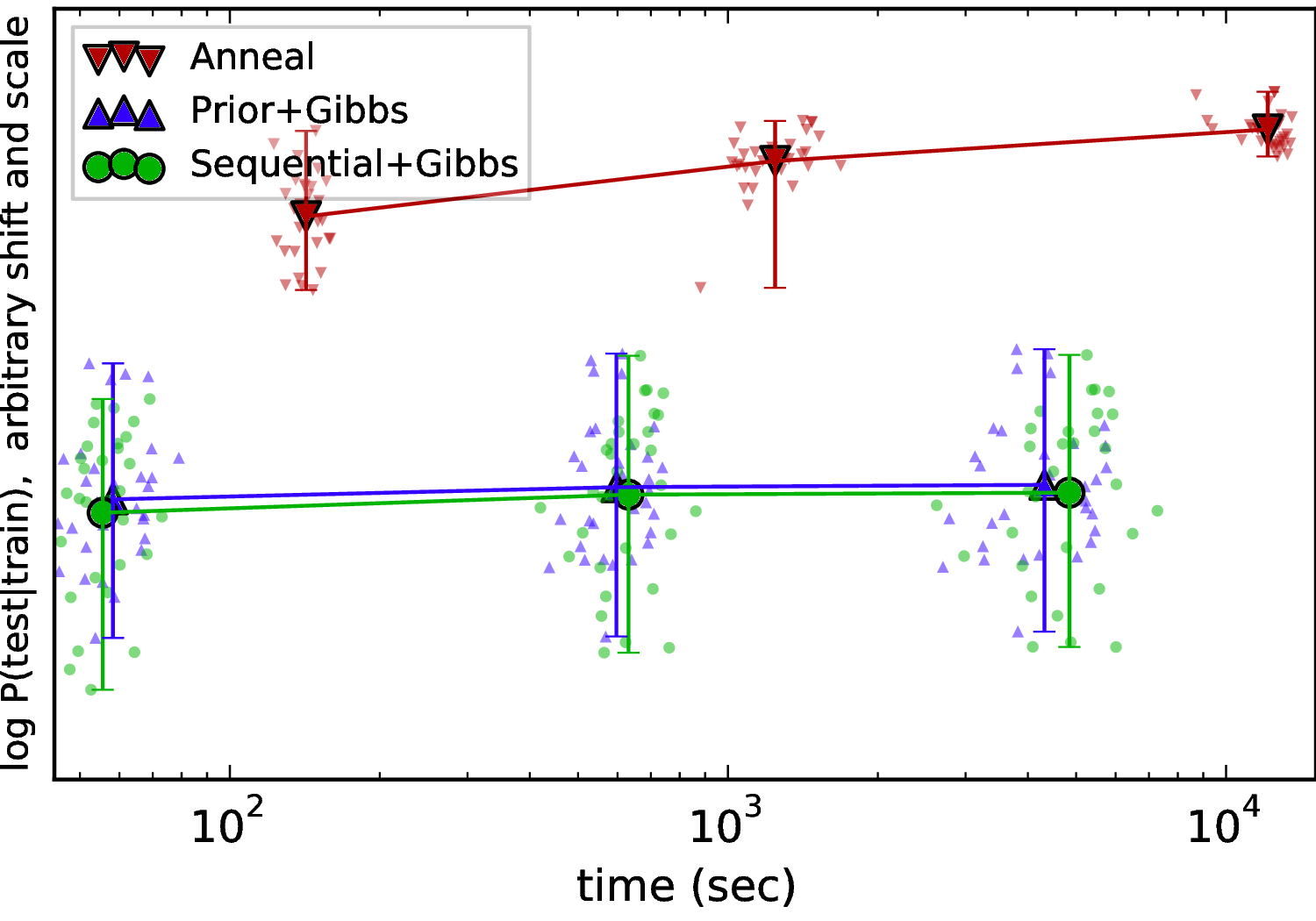} \\
\textsf{\small Network 1000000$\times$40} \\
\includegraphics[width=0.41\textwidth]{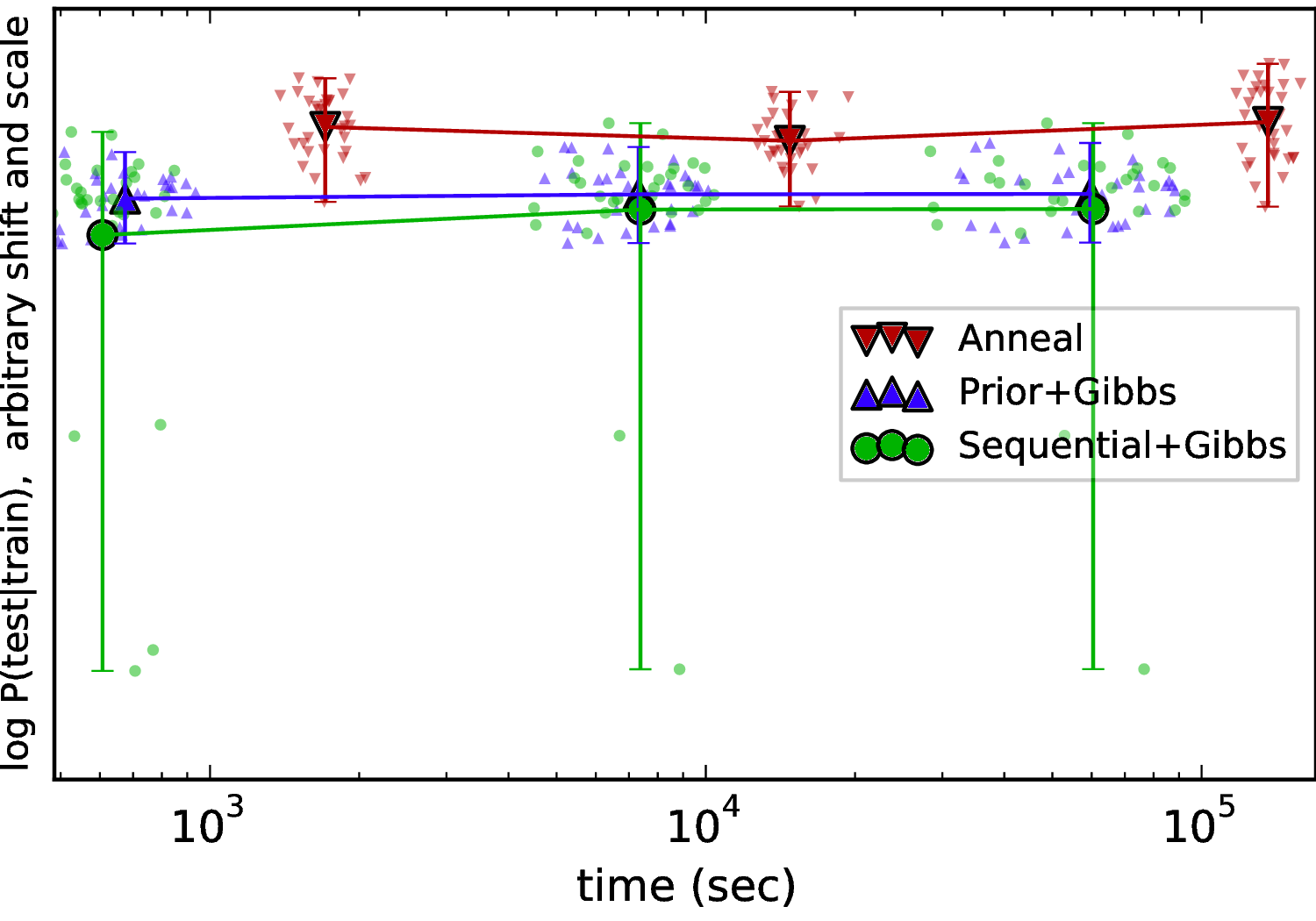}
\end{tabular} &
\begin{tabular}{c}
\textsf{\small Census 10000$\times$68} \\
\includegraphics[width=0.41\textwidth]{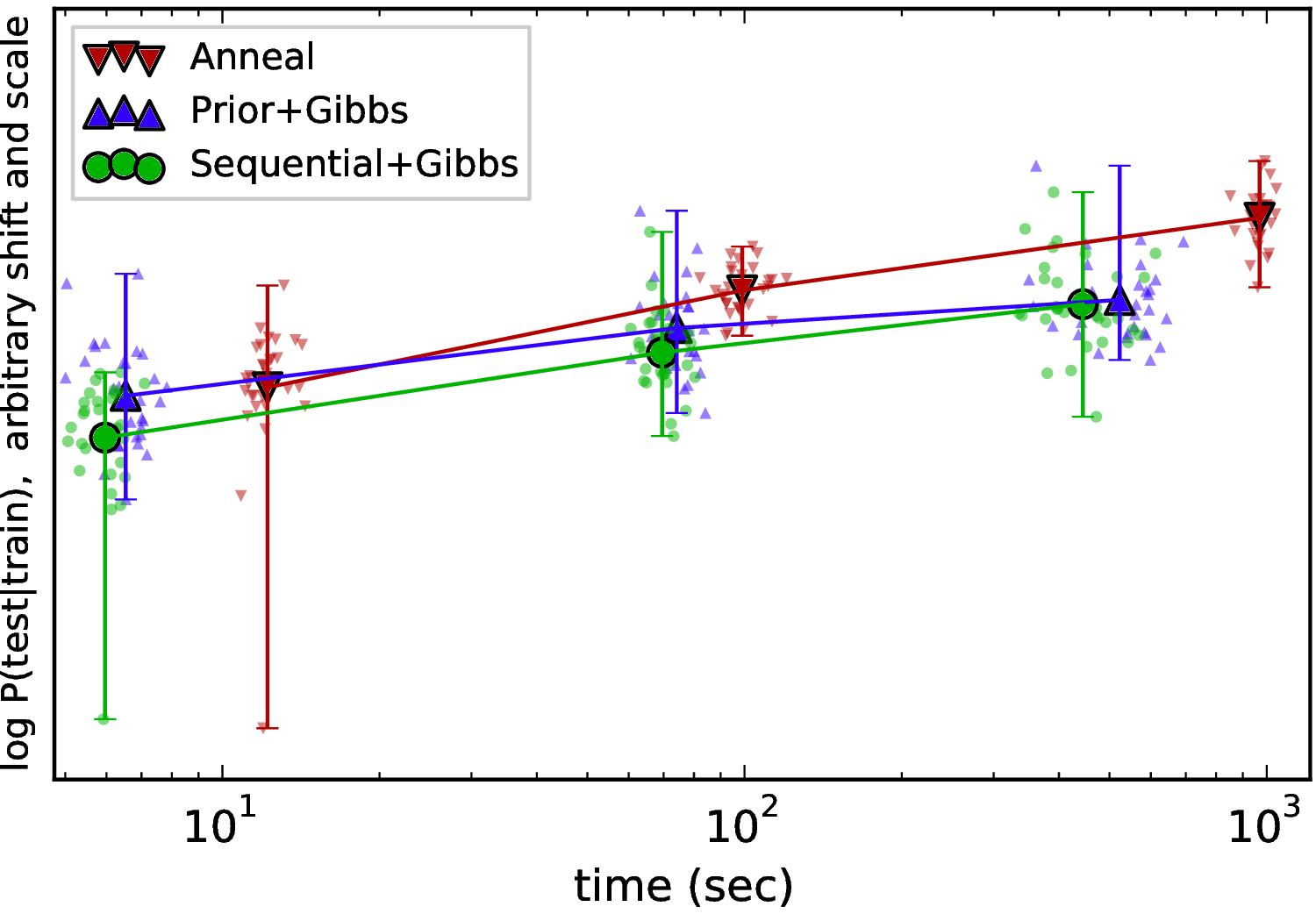} \\
\textsf{\small Census 100000$\times$68} \\
\includegraphics[width=0.41\textwidth]{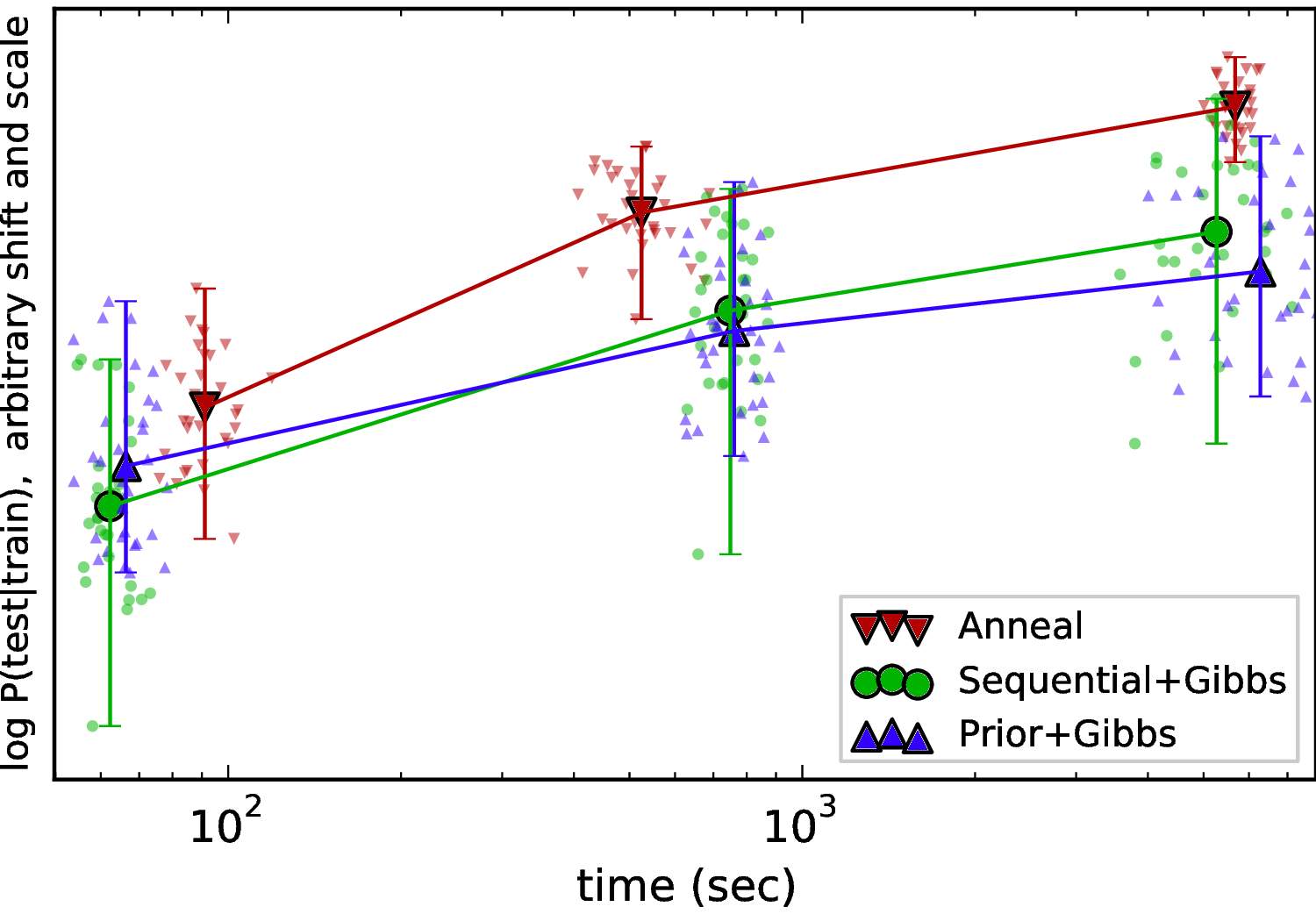} \\
\textsf{\small Census 1000000$\times$68} \\
\includegraphics[width=0.41\textwidth]{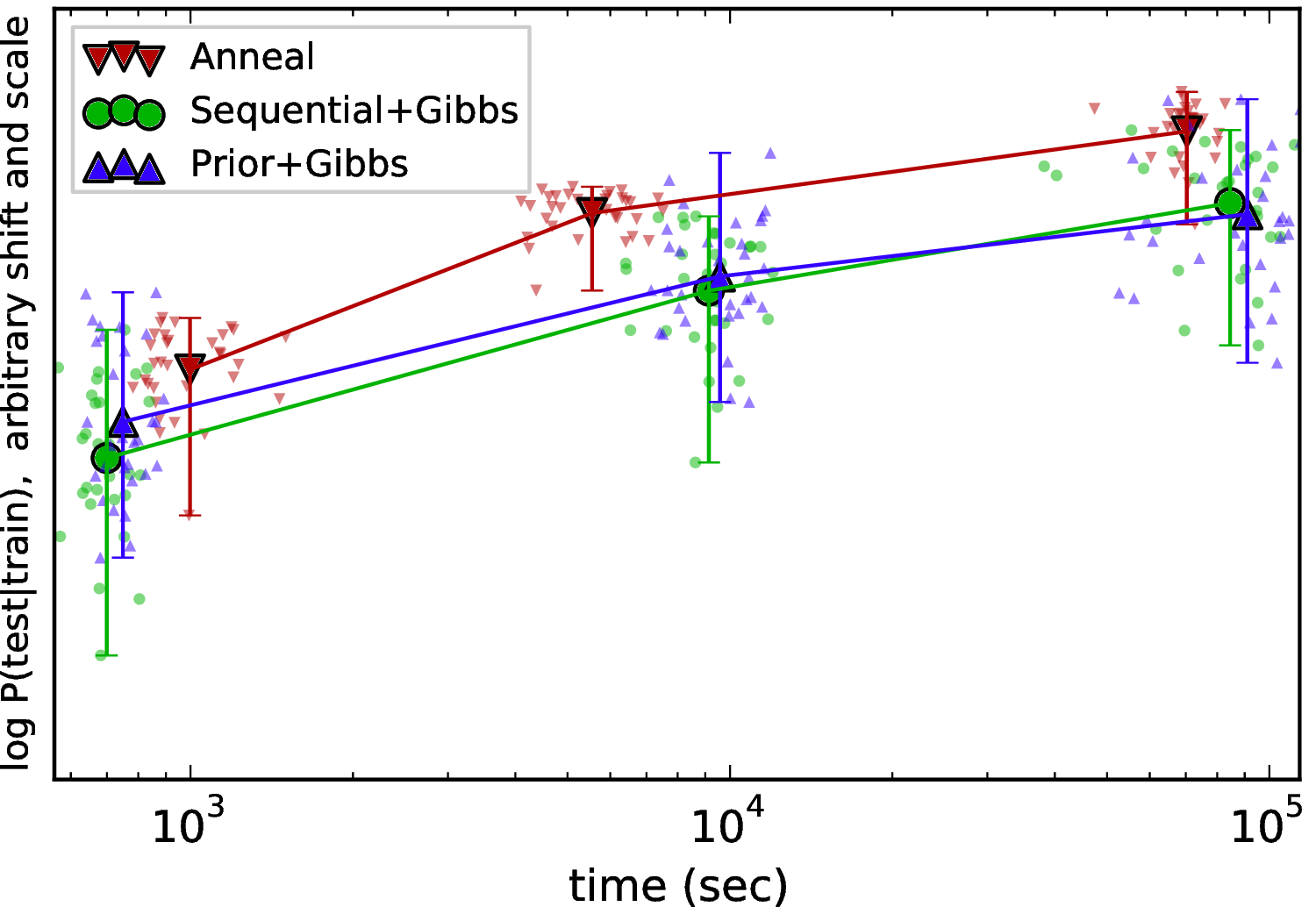}
\end{tabular}
\end{tabular}
\caption{
Crossvalidation scores of 3 inference strategies learning 7 datasets constrained to different bounds on wall clock time.
The Anneal strategy runs subsample annealing.
The Prior+Gibbs and Sequential+Gibbs strategies run full-data MCMC chains with prior and sequential initialization, resp.
}
\label{fig:scale}
\end{figure*}


Empirically, we find subsample annealing results in better crossvalidation score and much lower variance in crossvalidation score (i.e., with fewer outlying bad samples), in bounded wall clock time.
Figure~\ref{fig:scale} shows three algorithms being run for different time-bounds on all datasets.
We also find (not shown) that the feature partitions resulting from subsample annealing are much more consistent, despite having been learned at a strictly higher temperature than those learned with MCMC inference.

We had expected sequential initialization to outperform initialization from the prior, as in the above toy model, however empirically sequential initialization is no better and is less robust than initialization from the prior.
We believe that this is due to our deeper model, where at the time of sequential initialization, hyperparameter values have not been learned, and the blindly random hyperparameter values lead to a poor initialized state.
Subsample annealing addresses this by regularly performing hyperparameter inference while slowly adding data to the subsample.

\section{Related Work}
\label{sec:related}


We view subsample annealing as an addition to the class of MCMC methods \cite{welling2011bayesian, ahn2012bayesian, korattikara2013austerity, vandemeent_arxiv_2014} which allow for a tunable tradeoff between bias and inference speed advocated by Korattikara, Chen, and Welling \cite{korattikara2013austerity} in the context of stochastic gradient descent with minibatches.
As in our approach, they produce a single sample by running a schedule that transitions from fast-but-inaccurate to accurate-but-slow, much like simulated annealing.

Simulated annealing \cite{kirkpatrick1983optimization} is an optimization algorithm wherein a candidate solution follows Markov chain dynamics through a time-varying energy ($=-\log$ probability) landscape.
Whereas in simulated annealing the energy landscape is scaled by a time varying inverse-temperature factor $\beta(t)$, our subsample annealing algorithm scales the dataset by subsampling a portion $\beta(t)$ of data.
Whereas in optimization applications the temperature parameter is cooled down to zero to yield an (approximate) optimal/MAP solution, in Bayesian applications the temperature is instead cooled down to 1 to yield a posterior sample \cite{neal1993probabilistic}.
The analogy between subsample size and temperature has recently been employed in van de Meent, Paige, and Wood's \textit{subsample tempering} algorithms \cite{vandemeent_arxiv_2014}.

Goodman and Sokal \cite{goodman1989multigrid} develop a multiscale method called Multigrid Monte Carlo (MGMC), adapted from the multigrid method of solving PDEs.
Their MGMC approach avoids the slowdown of single-site Gibbs inference by performing inference on a hierarchy of representations at different levels of coarseness.
This approach is very similar to the ``horizontal compression'' phenomenon we saw in section~\ref{sec:analysis-local}.
Liu and Sabatti \cite{liu2000generalised} generalize the MGMC method to a wider range of models including ``nonparametric'' continuous time series.

Subsample annealing is related to other generalized annealing techniques, in particular the sequential buildup algorithm \cite{wong1995comment}, which is based on the observation that the annealed (or tempered) parameter need not correspond to a formal temperature, but may index any ``suitably overlapping'' collection of distributions.
A similar idea has been applied to the problem of generating proposal clusterings in a Metropolis-Hastings kernel for DPMMs \cite{dahl2005sequentially}.
Additionally, sequential Monte Carlo or particle filtering has been used in streaming inference of large datasets \cite{chopin2002sequential, ridgeway2002bayesian, del2006sequential}.

Finally our method is an addition to the growing number of subsample/minibatch approaches in both variational \cite{hoffman2012stochastic} and MCMC \cite{welling2011bayesian, ahn2012bayesian, korattikara2013austerity, vandemeent_arxiv_2014} inference.

\section{Discussion}


Practitioners have increasingly used ideas like subsample annealing to cope with large datasets, e.g., it is common to learn categories or hyperparameters from a subsampled dataset, then run simpler inference to categorize the full dataset.
We have provided a principled alternative to these heuristics, proving data-linear scaling in some models, and demonstrating improved inference on real datasets.

It is an open question when subsample annealing can offer speedup.
We have seen subsample annealing significantly improve inference quality in DPMM and Cross-Categorization models of some real datasets (as in Figure~\ref{fig:scale}), but on some datasets it provides no improvement.
We suspect this has to do with the shape of the energy landscape in different datasets.
In the language of section~\ref{sec:analysis-global}, subsampling offers the most improvement when the energy barriers $\delta$ are large, and the energy gaps $\gamma$ between between hypotheses are moderate --sufficiently small that many hypotheses are plausible%
, but sufficiently large relative to $\delta$ that when the chain is run hot enough to lower the barrier to a small $\beta\delta$, the gap $\beta\gamma$ is still observable.

Subsample annealing may also be sensitive to rare outliers that are missed early in the annealing schedule.
We suspect that the poor performance in the largest Network dataset above may be caused by rare ``network intrusion events''.
We hope to address outliers by using stratified subsampling, which favors diverse subsamples.

While the theoretical results of section~\ref{sec:analysis-local} indicate exponential speedup in a two-component mixture model, we suspect that nonparametric DPMMs allow faster mixing, so that subsample annealing may only provide a polynomial speedup.
However we suspect that subsample annealing will help mixing even more dramatically when learning structure, since structure learning intuitively has higher barriers between hypotheses.

\ifnum\statePaper=1{

\subsection{Acknowledgements}


We are grateful to the anonymous reviewers, Finale Doshi-Velez, and Beau Cronin for helpful comments.

}\fi

\bibliography{references}{}

\begin{thebibliography}{10}

\bibitem{ahn2012bayesian}
Sungjin Ahn, Anoop Korattikara, and Max Welling.
\newblock Bayesian posterior sampling via stochastic gradient fisher scoring.
\newblock {\em arXiv preprint arXiv:1206.6380}, 2012.

\bibitem{Bache+Lichman:2013}
K.~Bache and M.~Lichman.
\newblock {UCI} machine learning repository, 2013.
\newblock US Census Data (1990) Data Set.

\bibitem{Bottou08thetradeoffs}
Léon Bottou and Olivier Bousquet.
\newblock The tradeoffs of large scale learning.
\newblock In {\em IN: ADVANCES IN NEURAL INFORMATION PROCESSING SYSTEMS 20},
  pages 161--168, 2008.

\bibitem{chopin2002sequential}
Nicolas Chopin.
\newblock A sequential particle filter method for static models.
\newblock {\em Biometrika}, 89(3):539--552, 2002.

\bibitem{dahl2005sequentially}
David~B Dahl.
\newblock Sequentially-allocated merge-split sampler for conjugate and
  nonconjugate dirichlet process mixture models.
\newblock {\em Journal of Computational and Graphical Statistics}, 11, 2005.

\bibitem{del2006sequential}
Pierre Del~Moral, Arnaud Doucet, and Ajay Jasra.
\newblock Sequential monte carlo samplers.
\newblock {\em Journal of the Royal Statistical Society: Series B (Statistical
  Methodology)}, 68(3):411--436, 2006.

\bibitem{diaconis2008Gibbs}
Persi Diaconis, Kshitij Khare, and Laurent Saloff-Coste.
\newblock Gibbs sampling, exponential families and orthogonal polynomials.
\newblock {\em Statistical Science}, 23(2):151--178, 2008.

\bibitem{fink1997compendium}
Daniel Fink.
\newblock A compendium of conjugate priors.
\newblock Technical report, Montana State University, May 1997.

\bibitem{gardiner2009stochastic}
Crispin~W Gardiner.
\newblock {\em Stochastic methods}.
\newblock Springer Berlin, 2009.

\bibitem{ghahramani2013bayesian}
Zoubin Ghahramani.
\newblock Bayesian non-parametrics and the probabilistic approach to modelling.
\newblock {\em Philosophical Transactions of the Royal Society A: Mathematical,
  Physical and Engineering Sciences}, 371(1984), 2013.

\bibitem{goodman1989multigrid}
Jonathan Goodman and Alan~D Sokal.
\newblock Multigrid monte carlo method. conceptual foundations.
\newblock {\em Physical Review D}, 40(6):2035, 1989.

\bibitem{hettich99kddcup}
S~Hettich and S.~D. Bay.
\newblock {UCI} {KDD} archive, 1999.
\newblock KDD Cup 1999 dataset.

\bibitem{hoffman2012stochastic}
Matt Hoffman, David~M Blei, Chong Wang, and John Paisley.
\newblock Stochastic variational inference.
\newblock {\em arXiv preprint arXiv:1206.7051}, 2012.

\bibitem{kirkpatrick1983optimization}
Scott Kirkpatrick, D.~Gelatt Jr., and Mario~P Vecchi.
\newblock Optimization by simmulated annealing.
\newblock {\em science}, 220(4598):671--680, 1983.

\bibitem{korattikara2013austerity}
Anoop Korattikara, Yutian Chen, and Max Welling.
\newblock Austerity in mcmc land: Cutting the metropolis-hastings budget.
\newblock {\em arXiv preprint arXiv:1304.5299}, 2013.

\bibitem{liu2000generalised}
Jun~S Liu and Chiara Sabatti.
\newblock Generalised gibbs sampler and multigrid monte carlo for bayesian
  computation.
\newblock {\em Biometrika}, 87(2):353--369, 2000.

\bibitem{mansinghka2009cross}
Vikash~K Mansinghka, Eric Jonas, Cap Petschulat, Beau Cronin, Patrick Shafto,
  and Joshua~B Tenenbaum.
\newblock Cross-categorization: A method for discovering multiple overlapping
  clusterings.
\newblock In {\em Proc. of Nonparametric Bayes Workshop at NIPS}, volume 2009,
  2009.

\bibitem{neal1993probabilistic}
Radford~M Neal.
\newblock Probabilistic inference using markov chain monte carlo methods.
\newblock 1993.

\bibitem{neal2000markov}
Radford~M Neal.
\newblock Markov chain sampling methods for dirichlet process mixture models.
\newblock {\em Journal of computational and graphical statistics},
  9(2):249--265, 2000.

\bibitem{pitman1997two}
Jim Pitman and Marc Yor.
\newblock The two-parameter poisson-dirichlet distribution derived from a
  stable subordinator.
\newblock {\em The Annals of Probability}, 25(2):855--900, 1997.

\bibitem{ridgeway2002bayesian}
Greg Ridgeway and David Madigan.
\newblock Bayesian analysis of massive datasets via particle filters.
\newblock In {\em Proceedings of the eighth ACM SIGKDD international conference
  on Knowledge discovery and data mining}, pages 5--13. ACM, 2002.

\bibitem{teh2006hierarchical}
Yee~Whye Teh, Michael~I Jordan, Matthew~J Beal, and David~M Blei.
\newblock Hierarchical dirichlet processes.
\newblock {\em Journal of the american statistical association}, 101(476),
  2006.

\bibitem{vandemeent_arxiv_2014}
Jan-Willem van~de Meent, Brooks Paige, and Frank Wood.
\newblock {Tempering by Subsampling}.
\newblock {\em ArXiv e-prints}, 2014.

\bibitem{welling2011bayesian}
Max Welling and Yee~W Teh.
\newblock Bayesian learning via stochastic gradient langevin dynamics.
\newblock In {\em Proceedings of the 28th International Conference on Machine
  Learning (ICML-11)}, pages 681--688, 2011.

\bibitem{wennberg2008tracking}
John~E Wennberg, Elliott~S Fisher, David~C Goodman, and Jonathan~S Skinner.
\newblock Tracking the care of patients with severe chronic illness-the
  dartmouth atlas of health care 2008.
\newblock 2008.

\bibitem{wong1995comment}
Wing~Hung Wong.
\newblock Comment on bayesian computation and stochastic system by besag et.
  al.
\newblock {\em Statistical Science}, 10:52--53, 1995.

\end{thebibliography}
\bibliographystyle{plain}

\appendix
\clearpage
\newpage

\section{Proofs of clustering speedup}
\label{appendix:local}

\begin{proof}[Proof of Lemma~\ref{thm:continuum}]
Given portions $x$ and $y$ of red and blue balls, resp., in the left urn, consider the $2\times 2\times 2$ possible Gibbs moves: remove red/blue from left/right urn and replace in left/right urn.
For large data size $N$, data size changes very little after a single removal, so the add and remove steps decouple into differentials $dx_{\text{rem}}$, $dy_{\text{rem}}$, $dx_{\text{add}}$, and $dy_{\text{add}}$.
Compute the probabilities of each move; then compute the mean and variance in $x$ and, by red-blue symmetry, $y$:
\begin{align*}
&\mathbb E[dx_{\text{rem}}] = \frac 1 N \frac {1-2 x} 2
\\
&\mathbb V[dx_{\text{rem}}] = \frac 1 {N^2} \frac {x(1-x)} r
\\
&\mathbb E[dx_{\text{add}}] =
  \frac 1 {2N} \frac {n_2 r_1 - n_1 r_2} {n_2 r_1 + n_1 r_2}
  + \frac \alpha {N^2} \frac {2 n_1} {x r + (1-x) l}
\\
&\mathbb V[dx_{\text{add}}] = \frac 1 {N^2}
    \frac {r \; n_2 r_1 \; n_1 r_2} {(n_2 r_1 + n_1 r_2)^2}
\\&\quad+ \frac \alpha {N^3}
    \left[
      \frac {n_1 n_2} {(x n_1 + (1-x) n_2)^2} + \frac 2 {x n_1 + (1-x) n_2}
    \right]
\end{align*}
with lower-case intrinsic quantities defined as
\begin{align*}
  r_1 &= (\text{\#red on left}) / N &
  r_2 &= (\text{\#red on right}) / N \\
  b_1 &= (\text{\#blue on left}) / N &
  b_2 &= (\text{\#blue on right}) / N \\
  n_1 &= r_1 + b_1 &
  n_2 &= r_2 + b_2
\end{align*}
These moments comprise the $N$-scaled Fokker-Planck coefficients
\begin{align*}
  f &= N \begin{bmatrix}
         \mathbb E[dx_{\text{rem}} + dx_{\text{add}}] \\
         \mathbb E[dy_{\text{rem}} + dy_{\text{add}}]
       \end{bmatrix}
\\
  D &= N^2 \begin{bmatrix}
         \mathbb V[dx_{\text{rem}} + dx_{\text{add}}] & 0 \\
         0 & \mathbb V[dy_{\text{rem}} + dy_{\text{add}}]
       \end{bmatrix}
\end{align*}
By inspection these depend only on intrinsic quantities and the scaled hyperparameter $\frac \alpha N$.
\end{proof}

\begin{proof}[Proof of Theorem~\ref{thm:local}]
At fixed error bound $\epsilon$, the continuous dynamics is within $\epsilon$ of true dynamics by data size, say, $N_\epsilon$.
Thus at large data sizes, the MCMC dynamics is linear and mixing time is $T_{\text{cold}}=O\left(N^2\log(\epsilon)\right)$.
In a subsample annealing schedule $\beta(t)=t/T$, the subsample annealing dynamics at subsamples larger than $N_\epsilon$ is approximately time-scaled versions of the dynamics at full size $N$, so the effective schedule length is
$$
T_{\text{eff}}
\;=\; \int_{\frac {N_\epsilon} N T}^T \frac {dt} {\beta(t)^2}
\;=\; \frac 1 N\left[ \frac 1 {N_\epsilon} - \frac 1 N \right]
\;\ge\; \frac 2 {N N_\epsilon}
$$
Since effective time is inverse in data size, annealing mixes in time $T_{\text{anneal}}=O\left(N\log(\epsilon)\right)$.
\end{proof}

\section{Proofs of bimodal speedup}
\label{appendix:global}

\begin{proof}[Proof of Lemma~\ref{thm:dynamics}]
Consider a two-state system $\mathbf x = [x, 1-x]^T$ at energy levels $[N\gamma, 0]$.
The steady-state solution at temperature $\beta$ should be $\pi_beta := [\sigma(\beta\gamma N), \sigma(\beta\gamma N)]$, where $\sigma(t) = \frac 1 {1 + \exp(-t)}$ is the logistic sigmoid function.
In continuous time mixing, we think of the state briefly jumping on to an energy barrier of height $\beta\delta N$ then jumping back down according to $\pi_\beta$.
If the rate of jumping up to energy $\beta\delta N$ is $\exp(-\beta\delta N)$, then the dynamics is:
\[
  \frac {d\mathbf x} {dt}
  = \exp(-\beta\delta N)
    \left[
      \begin{bmatrix}
        \sigma(\beta\gamma N) & \sigma(\beta\gamma N) \\
        \sigma(-\beta\gamma N) & \sigma(-\beta\gamma N)
      \end{bmatrix}
      - \mathbf I
    \right]
    \mathbf x
\]
The first coordinate $x$ determines the state; expanding yields Equation~\ref{eqn:dynamics}.
\end{proof}

\begin{proof}[Proof of Theorem~\ref{thm:global}]
In this binary system the TVD of state $x$ from truth is $|x-x_{\text{true}}|=|x-\sigma(\gamma N)|$.
Now we seek asymptotic lower bounds on $T$ guaranteeing TVD$<\epsilon$.
To prove (a) observe that in cold inference ($\beta=1$), the system is linear homogeneous with eigenvalue $\exp(-N\delta)$.
To prove (b) we transform from time coordinates $t$ to ``natural'' coordinates
$$
\tau = \exp
  \left(
    \frac T {N\delta}
    \left[
      \exp\left( -\frac {N\delta t} T \right) -\exp\left( -N\delta \right)
    \right]
  \right),
$$
where,
assuming worst-case initial condition $x(0) = 0$,
the final state $x$ is a uniform integral
$$
  x = \int_0^1 \sigma(\beta(\tau) N\gamma)\,d\tau
$$
involving the transformed annealing schedule
$$
  \beta(\tau) = \frac {-1} {N\delta} \log
    \left(
    \exp(-N\delta) - \frac {N\delta} T \log(\tau)
    \right).
$$
Using the inequality $\sigma(\gamma)-\sigma(\beta\gamma)\le\exp(-\beta\gamma)$,
we can bound error by
\begin{equation}
\label{eqn:tvd-bound}
  \text{TVD} < \int_0^1 \exp(-\beta(\tau) N\gamma)\,d\tau
\end{equation}
Since the integrand $\exp(-\beta(\tau)N\gamma)$ is bounded in $(0,1)$, and $\beta(\tau)$ is increasing, Equation~\ref{eqn:tvd-bound} holds if $T$ is chosen large enough that $\exp(-\beta(\epsilon/2)>1$, for example if
$$
  T > \frac {N\delta \log\left(\frac 2 \epsilon\right)}
            {\left(\frac \epsilon 2\right)^{\frac \delta \gamma} - \exp(-N\delta)}
$$
or more conservatively, for any $K>1$, and sufficiently large $N$, 
$$
  T > KN\delta
      \log\left(\frac 2 \epsilon\right)
      \left(\frac 2 \epsilon \right)^{\frac \delta \gamma},
$$
whence the asymptotic bound.
\end{proof}

\end{document}